\documentclass{article}

\usepackage{amsmath}
\usepackage[ruled,linesnumbered]{algorithm2e}
\usepackage{amsthm}
\newtheorem{definition}{Definition}
\newtheorem{assumption}{Assumption}
\newtheorem{theorem}{Theorem}
\newtheorem{lemma}{Lemma}

\usepackage{array}
\usepackage{makecell}
\usepackage{multirow}
\usepackage{color}

\usepackage{PRIMEarxiv}

\usepackage[utf8]{inputenc} 
\usepackage[T1]{fontenc}    
\usepackage{hyperref}       
\usepackage{url}            
\usepackage{booktabs}       
\usepackage{amsfonts}       
\usepackage{nicefrac}       
\usepackage{microtype}      
\usepackage{lipsum}
\usepackage{fancyhdr}       
\usepackage{graphicx}       
\graphicspath{{media/}}     
\title{Learning universal approximations for partial differential equations with Physics-Informed Broad Learning System
}

\author{
    \parbox{\linewidth}{\centering
        Zhiwen Yu\textsuperscript{1,2} \quad
        Derong Yang\textsuperscript{1} \quad
        Liujian Zhang\textsuperscript{*3,2} \quad
        Kaixiang Yang\textsuperscript{1} \\[0.2cm]
        Peilin Zhan\textsuperscript{4} \quad
        Jianmin Lv\textsuperscript{1} \quad
        Jane You\textsuperscript{5} \quad
        C. L. Philip Chen\textsuperscript{1}
    } \\
    \vspace{0.4cm} 
    \parbox{\linewidth}{\centering \small 
        \textsuperscript{1}School of Computer Science and Engineering, South China University of Technology, Guangzhou 510006, China \\
        \textsuperscript{2}Peng Cheng Laboratory, Shenzhen 518000, China \\
        \textsuperscript{3}School of Future Technology, South China University of Technology, Guangzhou 510006, China \\
        \textsuperscript{4}School of Computer Science and Technology, Guangdong University of Technology, Guangzhou 510006, China \\
        \textsuperscript{5}Department of Industrial and Systems Engineering, The Hong Kong Polytechnic University
    }
}

\begin{document}
\maketitle
\maketitle

\begingroup
\renewcommand{\thefootnote}{*}
\footnotetext{\textbf{Corresponding author}: \texttt{202210192021@mail.scut.edu.cn} (Liujian Zhang) \\
}
\setcounter{footnote}{0} 
\endgroup

\begin{abstract}
Partial differential equations (PDEs) play a central role in modeling complex physical, biological, and engineering systems. While traditional numerical solvers are robust, they often incur prohibitive computational costs due to mesh dependencies, whereas recent Physics-Informed Neural Networks (PINNs) offer a mesh-free alternative but frequently suffer from slow convergence and optimization instability. To bridge this gap, this article proposes the Physics-Informed Broad Learning System (PIBLS), a novel backpropagation-free framework that reformulates PDE solving as a direct least-squares optimization. We improved an algorithm within this framework to handle nonlinear PDEs efficiently and provide a rigorous mathematical proof establishing the universal approximation property of PIBLS for these equations. Experiments on linear and nonlinear PDEs demonstrate that PIBLS is one to three orders of magnitude faster than conventional PINNs while achieving significantly higher solution accuracy. This framework provides a computationally efficient paradigm for scientific machine learning, offering a practical, high-speed alternative for real-time simulation and design optimization tasks.

\end{abstract}

\section{Introduction}
Partial differential equations (PDEs) are fundamental for modeling physical phenomena across various scientific and engineering fields. The primary challenge in their application lies in finding the solution, which is typically an unknown function representing a physical field, that satisfies the governing equation throughout a continuous domain and its associated boundary and initial conditions. Traditional numerical methods, such as the Finite Element Method (FEM)~\cite{FEM}, Finite Volume Method (FVM)~\cite{FVM}, and Finite Difference Method (FDM)~\cite{FDM}, have long been the standard for solving these equations. However, these methods rely heavily on the quality and density of spatial discretization, requiring very fine meshes to achieve high accuracy. For complex problems, this results in large computational costs and a significant burden on system resources~\cite{mesh_drawback1,mesh_drawback2,mesh_drawback3}.

In response to these limitations, deep learning-based solvers like Physics-Informed Neural Networks (PINNs) have emerged as a promising alternative~\cite{PINN}. By embedding governing physical laws directly into the loss function, PINNs can handle complex problems without requiring manual mesh generation~\cite{PIML_review1,PIML_review2,PIML_review3,PINN_highdimension,zhang, nazari, sarker, xing}. However, the practical use of PINNs is hindered by their reliance on deep architectures and slow gradient-based optimization techniques~\cite{adam,LBFGS,backpropagation}. This training process is notoriously complex, sensitive to hyperparameter tuning, and can suffer from pathological loss landscapes or vanishing gradients, which may prevent it from finding an accurate solution~\cite{PINN_drawback,PINN_drawback2,krishnapriyan2021characterizing}, creating a major bottleneck for their widespread adoption. To this end, a broader class of randomized neural network (RaNN) methods, represented by Physics-Informed Extreme Learning Machines (PIELM)~\cite{PIELM}, has been developed to bypass backpropagation and achieve significant computational speedups through closed-form solvers based on random feature mappings~\cite{pseudo_inverse}. Despite being exceptionally fast, these frameworks often lack the necessary representational power to accurately model complex solutions~\cite{ELM_drawback}.

Broad Learning System (BLS) is a novel neural network architecture introduced by Chen et al.~\cite{BLS}. Unlike traditional deep learning, which enhances performance by stacking multiple layers, BLS improves learning capacity through a broad expansion of feature and enhancement nodes within a shallow structure. Inspired by the random vector functional link neural network (RVFLNN)~\cite{rvflnn, rvflnn2}, only the output weights are trained using the Moore–Penrose pseudoinverse~\cite{Moore_Penrose_pseudoinverse}, which allows for fast, backpropagation-free learning. Theoretical analyses~\cite{BLS_universal_approximation} have established that BLS possesses universal approximation properties, granting it expressive power comparable to that of deep neural networks while achieving significantly faster training efficiency~\cite{BLS_speed,BLS_review,BLS_review2}. With its rapid training, strong performance, flexible and robust parameter selection, and backpropagation-free single-step computation, BLS has been widely applied and continuously improved across many research fields~\cite{zhang2025bls,yu2020bls,chen2025bls, zhong2025bls,yun2024bls,BLS_Incremental,BLS_Diagnosis}, demonstrating great potential for modeling complex dynamic systems.

In this paper, we propose the PIBLS, a novel framework that maintains exceptional speed while achieving a new level of precision that surpasses existing PDE solvers. The main contributions of this work are summarized as follows:

\begin{figure*}[!t]
    \centering
    \includegraphics[width=0.95\textwidth]{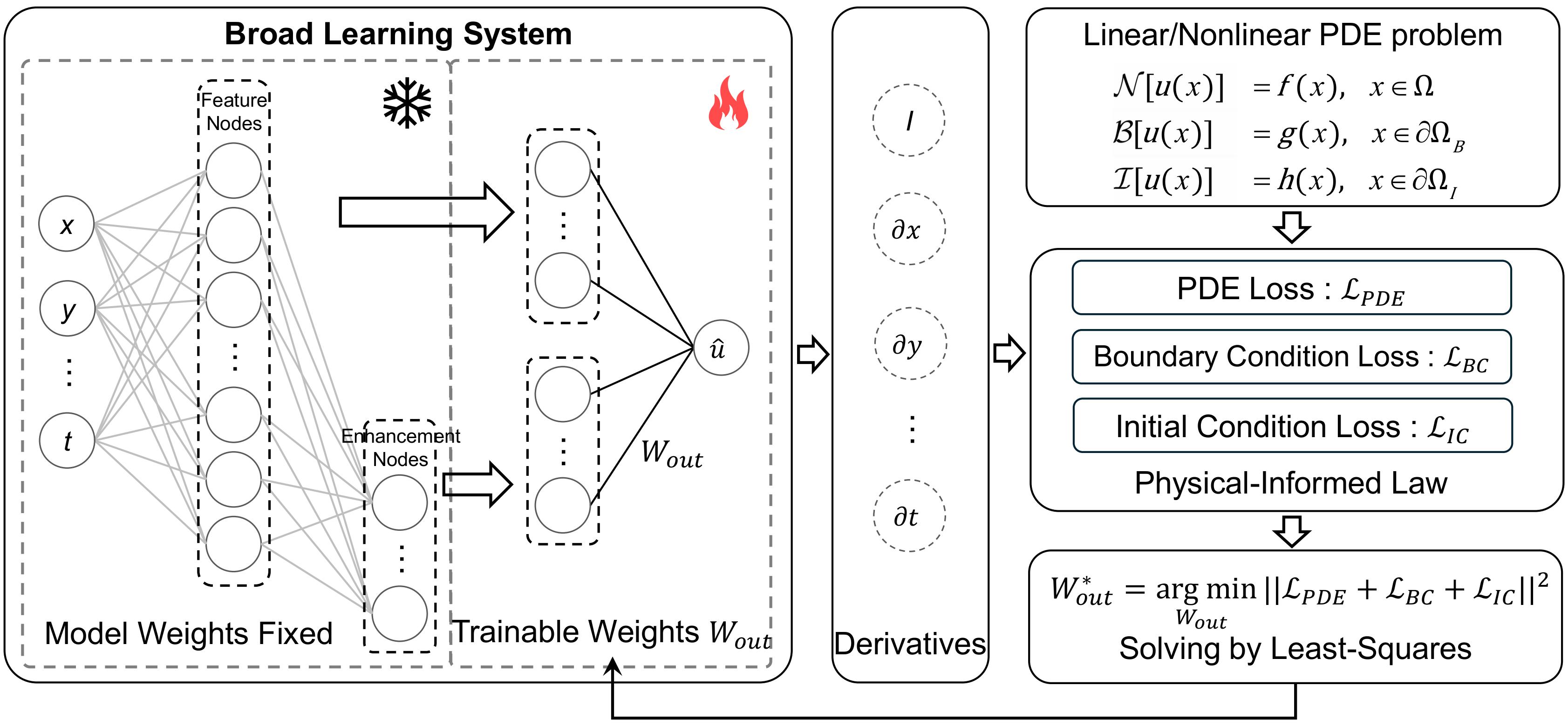}
    \caption{\textbf{Framework of the proposed Physics-Informed Broad Learning System (PIBLS).} The input coordinates are projected into the Broad Learning System (BLS). The BLS consists of randomly generated, fixed Feature Nodes and Enhancement Nodes, which form a basis. In this diagram, light grey lines represent the fixed weights associated with these nodes. The network output $\hat{u}$ is a linear combination of this basis, computed via the trainable output weights $\mathbf{W}_\text{out}$. Analytical derivatives of $\hat{u}$ are computed and used to formulate the residuals of the governing PDE ($\mathcal{N}[u(\boldsymbol{x})] = f(\boldsymbol{x})$), boundary conditions ($\mathcal{B}[u(\boldsymbol{x})] = g(\boldsymbol{x})$), and initial conditions ($\mathcal{I}[u(\boldsymbol{x})] = s(\boldsymbol{x})$). These residuals define the Physics-Informed Losses $\mathcal{L}_\text{PDE}$, $\mathcal{L}_\text{BC}$, and $\mathcal{L}_\text{IC}$. The optimal weights $\mathbf{W}_\text{out}^{\ast}$ are obtained by minimizing the total squared loss $\mathcal{L}_\text{PDE} + \mathcal{L}_\text{BC} + \mathcal{L}_\text{IC}$ via a least-squares method.}
    \label{fig:pibls_framework}
\end{figure*}
\begin{itemize}
    \item To the best of our knowledge, PIBLS is the first application of the BLS to solving PDEs, providing a new paradigm for efficient, backpropagation-free scientific computing. By leveraging mapped feature and enhancement nodes as a robust basis, this design offers a architecture capable of modeling complex physical fields.

    \item A unique solver strategy is developed to address the computational challenges of governing equations. We reformulate the problem as a direct least-squares optimization. For linear equations, an analytical solution is derived for rapid computation; for nonlinear systems, an enhanced nonlinear least-squares perturbation algorithm is introduced to ensure stable convergence.

    \item We provide a rigorous mathematical proof of the universal approximation property of PIBLS, theoretically guaranteeing its capability to approximate PDE solutions.
    
    \item Extensive experiments demonstrate that PIBLS is one to three orders of magnitude faster than conventional PINNs while achieving significantly higher solution accuracy, establishing it as a computationally efficient alternative to traditional numerical methods.
\end{itemize}

\section{Methodology}
This section details mathematical formulation of the PIBLS. We first define the network architecture and derive its closed-form analytical derivatives. Subsequently, we formulate the PDE solving task as a least-squares optimization problem, detailing specific strategies for linear and nonlinear systems. Finally, we provide a mathematical proof establishing the universal approximation capability of the proposed framework.

\subsection{The PIBLS structure}

Fig.~\ref{fig:pibls_framework} illustrates the overall framework of PIBLS, which is built upon the BLS architecture~\cite{BLS}. Unlike deep networks that learn features hierarchically~\cite{deep_vs_shallow_PINN1,deep_vs_shallow_PINN2}, PIBLS constructs an expressive feature representation by expanding the network's width rather than its depth.

Specifically, the architecture establishes a mapping from the input domain $\boldsymbol{x} \in \mathbb{R}^D$ to the solution space using two sets of basis functions: feature nodes and enhancement nodes.

First, the input $\boldsymbol{x}$ is projected into $n$ groups of feature mappings, where each group consists of $k$ nodes. The $i$-th group of feature nodes $z_i(\boldsymbol{x})$ is generated via a nonlinear transformation:
\begin{equation}
    \label{eq:feature_nodes}
    z_i(\boldsymbol{x}) = \phi(\boldsymbol{x} \mathbf{W}_{e_i} + \boldsymbol{\beta}_{e_i}), \quad i = 1, \dots, n
\end{equation}
where $\mathbf{W}_{e_i} \in \mathbb{R}^{D \times k}$ and $\boldsymbol{\beta}_{e_i} \in \mathbb{R}^k$ are weights and biases corresponding to the $i$-th group. These parameters are initialized from a uniform distribution and remain fixed. In this work, the activation function $\phi(\cdot)$ is the hyperbolic tangent ($\tanh$). Subsequently, all $n$ feature groups are concatenated to form the complete feature layer representation $\mathbf{Z} \in \mathbb{R}^{nk}$:
\begin{equation}
    \label{eq:feature_groups}
    \mathbf{Z} \equiv [z_1(\boldsymbol{x}), z_2(\boldsymbol{x}), \dots, z_n(\boldsymbol{x})].
\end{equation}

To promote a sparse and compact representation, the feature mapping weights and biases ($\mathbf{W}_{e_i}, \boldsymbol{\beta}_{e_i}$) can optionally be fine-tuned using a sparse autoencoder at this stage.

The feature representation is further enriched by utilizing this aggregated feature set $\mathbf{Z}$ to generate $m$ groups of enhancement nodes, where each group consists of $q$ nodes. The $j$-th group of enhancement nodes $h_j$ is generated similarly:
\begin{equation}
    \label{eq:enhancement_nodes}
    h_j(\boldsymbol{x}) = \xi(\mathbf{Z} \mathbf{W}_{h_j} + \boldsymbol{\beta}_{h_j}), \quad j = 1, \dots, m
\end{equation}
where $\mathbf{W}_{h_j} \in \mathbb{R}^{nk \times q}$ and $\boldsymbol{\beta}_{h_j} \in \mathbb{R}^q$ are fixed random weights and biases connecting the feature nodes to the enhancement nodes. Similarly, the tanh function is adopted for $\xi(\cdot)$. The full set of enhancement nodes forms the concatenated representation $\mathbf{H} \in \mathbb{R}^{mq}$:
\begin{equation}
    \label{eq:enhancement_groups}
    \mathbf{H} \equiv [h_1(\boldsymbol{x}), h_2(\boldsymbol{x}), \dots, h_m(\boldsymbol{x})].
\end{equation}

The basis matrix is formed by concatenating all feature and enhancement node groups:
\begin{equation}
\label{eq:final_basis}
    \mathbf{A} = [\mathbf{Z} \mid \mathbf{H}].
\end{equation}

The network's output $\hat{u}(\boldsymbol{x})$ serves as the surrogate solution to the PDE. It is computed as a linear combination of the final basis via the trainable output weights $\mathbf{W}_\text{out}$:
\begin{equation}
\label{eq:output}
    \hat{u}(\boldsymbol{x}) = \mathbf{A} \mathbf{W}_\text{out}.
\end{equation}

A key advantage of this architecture is that only the output weights $\mathbf{W}_\text{out}$ are trained. All internal network parameters remain fixed after the initialization phase. Accordingly, the total number of nodes $N_{\text{nodes}}$ is defined as the sum of the feature and enhancement nodes, given by $N_{\text{nodes}} = nk + mq$.

\subsection{Physics-informed least-squares formulation}
To ensure the network output satisfies the governing PDE, we embed the physical laws into the solution process. We consider a general, time-dependent differential equation governed by three sets of constraints:
\begin{equation}
    \begin{aligned}
        \mathcal{N}[u(\boldsymbol{x})] &= f(\boldsymbol{x}), \quad \boldsymbol{x} \in \Omega, \\
        \mathcal{B}[u(\boldsymbol{x})] &= g(\boldsymbol{x}), \quad \boldsymbol{x} \in \partial \Omega_B, \\
        \mathcal{I}[u(\boldsymbol{x})] &= s(\boldsymbol{x}), \quad \boldsymbol{x} \in \partial \Omega_I,
    \end{aligned}
\end{equation}
where $\mathcal{N}[\cdot]$, $\mathcal{B}[\cdot]$, and $\mathcal{I}[\cdot]$ are the differential, boundary, and initial condition operators, respectively. The terms $f(\boldsymbol{x})$, $g(\boldsymbol{x})$, and $s(\boldsymbol{x})$ correspond to the source term, boundary conditions, and initial conditions. The domain $\Omega$ represents the spatio-temporal interior, $\partial \Omega_B$ the spatial boundary, and $\partial \Omega_I$ the initial time boundary.

The central principle of the PIBLS framework is to find the single optimal output weight matrix $\mathbf{W}_\text{out}$ that minimizes the total residual across the problem's domain. To achieve this, we discretize the training points for the interior $\mathcal{X}_f = \{\boldsymbol{x}_i\}_{i=1}^{N_f} \subset \Omega$, the spatial boundary points $\mathcal{X}_\text{bc} = \{\boldsymbol{x}_j\}_{j=1}^{N_\text{bc}} \subset \partial \Omega_B$, and the initial time boundary points $\mathcal{X}_\text{ic} = \{\boldsymbol{x}_k\}_{k=1}^{N_\text{ic}} \subset \partial \Omega_I$. For any given point in these sets, the residuals are defined as:
\begin{equation}
    \label{eq:residuals}
    \begin{aligned}
        r_{\mathcal{N}}(\boldsymbol{x}_i) &= \mathcal{N}[\hat{u}(\boldsymbol{x}_i)] - f(\boldsymbol{x}_i), \quad i = 1, \dots, N_f, \\
        r_{\mathcal{B}}(\boldsymbol{x}_j) &= \mathcal{B}[\hat{u}(\boldsymbol{x}_j)] - g(\boldsymbol{x}_j), \quad j = 1, \dots, N_\text{bc}, \\
        r_{\mathcal{I}}(\boldsymbol{x}_k) &= \mathcal{I}[\hat{u}(\boldsymbol{x}_k)] - s(\boldsymbol{x}_k), \quad k = 1, \dots, N_\text{ic}.
    \end{aligned}
\end{equation}

For steady-state (time-independent) problems, the initial condition $r_{\mathcal{I}}$ and $s(\boldsymbol{x})$ do not exist.

\begin{algorithm}[!t]
\caption{PIBLS Linear Least Squares Solver}
\label{alg:linear_solver}
\SetAlgoLined
\DontPrintSemicolon
\SetKwComment{tcp}{}{}%
\SetKwInOut{Input}{Input}\SetKwInOut{Output}{Output}
\Input{
    Point sets $\mathcal{X} = \{ \mathcal{X}_f, \mathcal{X}_\text{bc}, \mathcal{X}_\text{ic} \}$, number of feature groups $n$, number of enhancement groups $m$, initialization range $R_m$
}
\Output{
    The optimal output weights $\mathbf{W}_\text{out}$
}
\For{$i \leftarrow 1$ \KwTo $n$}{
    $\mathbf{W}_{e_i}, \boldsymbol{\beta}_{e_i} \leftarrow \text{Random}(-R_m, R_m)$ \;
    Calculate $z_i(\mathcal{X})$ using~\eqref{eq:feature_nodes} \;
}
Get $\mathbf{Z}$ using~\eqref{eq:feature_groups} \;

\For{$j \leftarrow 1$ \KwTo $m$}{
    $\mathbf{W}_{h_j}, \boldsymbol{\beta}_{h_j} \leftarrow \text{Random}(-R_m, R_m)$ \;
    Calculate $h_j(\mathcal{X})$ using~\eqref{eq:enhancement_nodes} \;
}
Get $\mathbf{H}$ using~\eqref{eq:enhancement_groups} \;

Assemble $\mathbf{A}_\text{sys}, \mathbf{T}$ using~\eqref{eq:assembly}

$\mathbf{W}_\text{out} \leftarrow \text{SolveLeastSquares}(\mathbf{A}_\text{sys} \mathbf{W} = \mathbf{T})$ \;

\Return $\mathbf{W}_\text{out}$
\end{algorithm}

\subsection{Analytical derivatives of the network}
Instead of relying on automatic differentiation, the PIBLS framework evaluates the residuals by directly computing derivatives. As defined in~\eqref{eq:output}, the network output is a linear combination of basis functions in $\mathbf{A}$. Let $\boldsymbol{x} = [x_1, \dots, x_D]$ denote the input coordinates. Since the output weights $\mathbf{W}_\text{out}$ are independent of the input coordinates, the $p$-th order partial derivative with respect to the $d$-th coordinate $x_d$ is:
\begin{equation}
    \frac{\partial^p \hat{u}}{\partial x_d^p} = \left( \frac{\partial^p \mathbf{A}}{\partial x_d^p} \right) \mathbf{W}_\text{out}.
\end{equation}

Consequently, the problem reduces to differentiating the basis matrix $\mathbf{A}$. Based on~\eqref{eq:final_basis}, the derivative is formed by concatenating the derivatives of the feature and enhancement blocks:
\begin{equation}
    \frac{\partial^p \mathbf{A} }{\partial x_d^p} = \left[ \frac{\partial^p \mathbf{Z}}{\partial x_d^p} \mid \frac{\partial^p \mathbf{H}}{\partial x_d^p} \right].
\end{equation}

We first derive the analytical derivatives for each individual feature node group $z_i$ to form $\mathbf{Z}$, and subsequently determine the derivatives for each enhancement node group $h_j$ to form $\mathbf{H}$.

Recalling that the definition of feature nodes in~\eqref{eq:feature_nodes}, the partial derivative of this linear argument with respect to $x_d$ corresponds to the $d$-th row of the weight matrix $\mathbf{W}_{e_i}$. By applying the chain rule, the $p$-th order derivative of the feature node output is:
\begin{equation}
    \frac{\partial^p z_i}{\partial x_d^p} = \phi^{(p)}(\boldsymbol{x} \mathbf{W}_{e_i} + \boldsymbol{\beta}_{e_i}) \odot \left( [(\mathbf{W}_{e_i})_{d, \cdot}]^{\odot p} \right), \quad i = 1, \dots, n
\end{equation}
where $\phi^{(p)}$ is the $p$-th derivative of the activation function, $\odot$ denotes the Hadamard product, and $(\cdot)^{\odot p}$ denotes the element-wise power. 
The derivative of the complete feature matrix $\mathbf{Z}$ is then obtained by concatenating these group derivatives: 
\begin{equation}
    \frac{\partial^p \mathbf{Z}}{\partial x_d^p} = \left[ \frac{\partial^p z_1}{\partial x_d^p}, \frac{\partial^p z_2}{\partial x_d^p}, \dots, \frac{\partial^p z_n}{\partial x_d^p} \right].
\end{equation}

The derivatives of the enhancement nodes are computed similarly via nested chain rules. To simplify the notation, we define the linear pre-activation of the $j$-th enhancement group as $\tilde{h}_j$:
\begin{equation}
    \tilde{h}_j = \mathbf{Z}\mathbf{W}_{h_j} + \boldsymbol{\beta}_{h_j}, \quad j = 1, \dots, m
\end{equation}
The output of the $j$-th group is then $h_j(\boldsymbol{x}) = \xi(\tilde{h}_j(\boldsymbol{x}))$. Applying the chain rule, the first-order derivative with respect to $x_d$ is:
\begin{equation}
    \frac{\partial h_j}{\partial x_d} 
    = \xi'(\tilde{h}_j) \odot \frac{\partial \tilde{h}_j}{\partial x_d}
    = \xi'(\tilde{h}_j) \odot \left( \frac{\partial \mathbf{Z}}{\partial x_d} \mathbf{W}_{h_j} \right).
\end{equation}

Higher-order derivatives follow the standard product rule. Consequently, the second-order derivative is compactly expressed as:
\begin{equation}
    \frac{\partial^2 h_j}{\partial x_d^2} = 
    \xi''(\tilde{h}_j) \odot \left( \frac{\partial \mathbf{Z}}{\partial x_d} \mathbf{W}_{h_j} \right)^{\odot 2} 
    + 
    \xi'(\tilde{h}_j) \odot \left( \frac{\partial^2 \mathbf{Z}}{\partial x_d^2} \mathbf{W}_{h_j} \right).
\end{equation}

Since $\phi$ and $\xi$ are smooth functions, all derivatives can be computed exactly. These analytical expressions allow PIBLS to efficiently compute the exact matrices for the differential operators required for the least-squares formulation, making the framework both fast and accurate.

\subsection{Solution of linear PDEs}
For linear PDEs, the solution process simplifies fundamentally. Since the network output in~\eqref{eq:output} is linear with respect to $\mathbf{W}_\text{out}$, applying any linear operator $\mathcal{O} \in \{\mathcal{N}, \mathcal{B}, \mathcal{I}\}$ decouples the weights from the basis matrix $\mathbf{A}$:
\begin{equation}
\label{eq:linear_op}
    \mathcal{O}[\hat{u}(\boldsymbol{x})] = \mathcal{O}[\mathbf{A} \mathbf{W}_\text{out}] = \mathcal{O}[\mathbf{A}] \mathbf{W}_\text{out}.
\end{equation}

To determine the weights, we substitute~\eqref{eq:linear_op} into the residual definitions in~\eqref{eq:residuals} and enforce the condition that they vanish at the training points. This allows us to assemble the constraints directly into a unified linear system:
\begin{equation}
    \begin{aligned}
        (\mathcal{N}[\mathbf{A}_f]) \mathbf{W}_\text{out} &= f(\mathcal{X}_f), \\
        (\mathcal{B}[\mathbf{A}_\text{bc}]) \mathbf{W}_\text{out} &= g(\mathcal{X}_\text{bc}), \\
        (\mathcal{I}[\mathbf{A}_\text{ic}]) \mathbf{W}_\text{out} &= s(\mathcal{X}_\text{ic}).
    \end{aligned}
\end{equation}

Here, the terms $\mathcal{N}[\mathbf{A}_f]$, $\mathcal{B}[\mathbf{A}_\text{bc}]$, and $\mathcal{I}[\mathbf{A}_\text{ic}]$ represent the matrices resulting from applying the differential, boundary, and initial condition operators to the basis functions and evaluating them at their respective point sets. Correspondingly, $f(\mathcal{X}_f)$, $g(\mathcal{X}_\text{bc})$, and $s(\mathcal{X}_\text{ic})$ represent the target matrix derived from the source term, boundary conditions, and initial conditions. Note that for steady-state problems, the initial condition do not exist.

We construct a single linear system by vertically stacking the matrix equations evaluated on the respective sets~\cite{least_squares}:
\begin{equation}
    \mathbf{A}_\text{sys} \mathbf{W}_\text{out} = \mathbf{T}
\end{equation}
where the system matrix $\mathbf{A}_\text{sys}$ and the target matrix $\mathbf{T}$ are composed of:
\begin{equation}
    \label{eq:assembly}
    \mathbf{A}_\text{sys} = 
    \begin{bmatrix} 
        \mathcal{N}[\mathbf{A}_f] \\ 
        \mathcal{B}[\mathbf{A}_\text{bc}] \\ 
        \mathcal{I}[\mathbf{A}_\text{ic}] 
    \end{bmatrix}, 
    \quad
    \mathbf{T} = 
    \begin{bmatrix} 
        f(\mathcal{X}_f) \\ g(\mathcal{X}_\text{bc}) \\ s(\mathcal{X}_\text{ic}) 
    \end{bmatrix}.
\end{equation}

This system is solved directly for the optimal weights $\mathbf{W}_\text{out}$ as detailed in Algorithm~\ref{alg:linear_solver}, using a standard linear least-squares method~\cite{pseudo_inverse}, thereby avoiding iterative gradient-based training entirely.

\begin{algorithm}[!t]
\caption{PIBLS Enhanced NLSQ-perturb Solver}
\label{alg:nlsq_perturb}
\SetAlgoLined
\DontPrintSemicolon
\SetKwComment{tcp}{}{}%
\SetKwInOut{Input}{Input}\SetKwInOut{Output}{Output}
\Input{
    Basis system matrix $\mathbf{A}_\text{sys}$, linearized target $\mathbf{T}_\text{lin}$, residual function $\boldsymbol{r}(\mathbf{W})$, Jacobian function $J(\mathbf{W})$, perturbation loops $k_\text{max}$, perturbation size $\delta$, tolerance $\mathcal{L}_\text{tol}$
}
\Output{
    The best solution $\mathbf{W}_\text{out}$
}

$\mathbf{W}_0 \leftarrow \text{SolveLeastSquares}(\mathbf{A}_{\text{sys\_lin}} \mathbf{W} = \mathbf{T}_\text{lin})$

$(\mathbf{W}_\text{out}, \mathcal{L}_\text{best}) \leftarrow \text{SolveNonlinear}(\boldsymbol{r}, J, \mathbf{W}_0)$

\If{$\mathcal{L}_\text{best} < \mathcal{L}_\text{tol}$}{
    \Return $\mathbf{W}_\text{out}$
}

\For{$k \leftarrow 1$ \KwTo $k_\text{max}$}{
    $\Delta \mathbf{W} \leftarrow \text{Random}(-\delta, \delta)$ \;

    $\mathbf{W}_{\text{guess}} \leftarrow \mathbf{W}_\text{out} + \Delta \mathbf{W}$

    $(\mathbf{W}_\text{new}, \mathcal{L}_\text{new}) \leftarrow \text{SolveNonlinear}(\boldsymbol{r}, J, \mathbf{W}_{\text{guess}})$

    \If{$\mathcal{L}_\text{new} < \mathcal{L}_\text{best}$ }
    { 
        $\mathbf{W}_\text{out} \leftarrow \mathbf{W}_\text{new}$ \;
        $\mathcal{L}_\text{best} \leftarrow \mathcal{L}_\text{new}$ \;
    }
    \If{$\mathcal{L}_\text{best} < \mathcal{L}_\text{tol}$}{
        \Return $\mathbf{W}_\text{out}$ 
    }
}

\Return $\mathbf{W}_\text{out}$
\end{algorithm}

\subsection{Solution of nonlinear PDEs}
To solve the nonlinear PDEs, we aim to find the weights $\mathbf{W}_\text{out}$ that drive the residuals defined in~\eqref{eq:residuals} to zero. Since the governing operators are nonlinear, a direct linear solution is not feasible. Instead, we formulate a NLSQ optimization task by directly aggregating the squared norms of these residual terms. The objective function is thus constructed as:
\begin{equation}
    \mathop{\mathrm{arg\,min}}\limits_{\mathbf{W}_\text{out}}
    \, \mathcal{L}(\mathbf{W}_\text{out}) = \|\boldsymbol{r}_{\mathcal{N}}(\mathcal{X}_f)\|_2^2 + \|\boldsymbol{r}_{\mathcal{B}}(\mathcal{X}_\text{bc})\|_2^2 + \|\boldsymbol{r}_{\mathcal{I}}(\mathcal{X}_\text{ic})\|_2^2,
\end{equation}
where the initial condition term $\boldsymbol{r}_{\mathcal{I}}(\mathcal{X}_\text{ic})$ is included only for time-dependent problems.

To solve this NLSQ problem, we employ the Trust-Region Reflective (TRF) algorithm~\cite{TRF}, which relies on the Jacobian matrix of the residuals $J(\mathbf{W}_\text{out})$. We compute this Jacobian efficiently by leveraging the exact analytical derivatives derived in the preceding section. Although the governing equation is nonlinear, the network approximation $\hat{u}$ itself remains linear with respect to the weights. Therefore, applying the chain rule yields an efficient analytical form:
\begin{equation}
J(\mathbf{W}_\text{out}) = \frac{\partial \boldsymbol{r}}{\partial \mathbf{W}_\text{out}} = \frac{\partial \boldsymbol{r}}{\partial \hat{u}} \frac{\partial \hat{u}}{\partial \mathbf{W}_\text{out}} = \frac{\partial \boldsymbol{r}}{\partial \hat{u}} \mathbf{A},
\end{equation}
where $\frac{\partial \boldsymbol{r}}{\partial \hat{u}}$ represents the variation of the differential operator with respect to the solution field.

To ensure robust convergence, we employ an Enhanced NLSQ-perturb algorithm (Algorithm~\ref{alg:nlsq_perturb}). Adapted from~\cite{LocalELM}, our approach integrates a deterministic local search with a stochastic global perturbation mechanism. We first utilize a physics-informed initialization by solving a linearized approximation of the PDE ($\mathcal{N}_\text{lin} \approx \mathcal{N}$) to obtain a high-quality initial guess $\mathbf{W}_0$. From this starting point, the TRF method reduces the computational overhead of the initial search phase. If the solution stagnates in a local minimum, a stochastic perturbation loop applies random noise to the weights, allowing the solver to escape and explore the global parameter space.

\section{Universal approximation capability of PIBLS}
A rigorous theoretical foundation is essential to substantiate the capability of PIBLS in solving complex PDEs. 
We study the approximation properties of PIBLS within Sobolev spaces $H^s(\Omega)$ for $s\in\{1,2\}$.
The notation is similar to Theorem 1 of~\cite{BLS}
. For theoretical convenience, we treat the feature layer as a flattened collection of $nk$ nodes parameterized by weight vectors $\boldsymbol{w}_{e_i}$.

\begin{definition}[Sobolev Space Framework]
Let $\Omega\subset\mathbb{R}^D$ be compact with Lipschitz boundary. 
The target solution $u^*$ lies in $H^s(\Omega)$ equipped with the norm
\begin{equation*}
\|u\|_{H^s(\Omega)}
= \left( \sum_{|\alpha|\le s} \|\partial^\alpha u\|_{L^2(\Omega)}^2 \right)^{1/2}.
\end{equation*}
\end{definition}

\begin{definition}[Functional PIBLS Representation]
Under the PIBLS architecture, the approximation $u:\Omega\to\mathbb{R}$ is written as
\begin{equation*}
u_{n,m}(\boldsymbol{x})=u^{\mathbf{Z}}_n(\boldsymbol{x})+u_m^{\mathbf{H}}(\boldsymbol{x}),
\end{equation*}
where
\begin{align*}
u^{\mathbf{Z}}_n(\boldsymbol{x})
 = \sum_{i=1}^{nk} w_i\,\phi\!\left(\boldsymbol{x} \boldsymbol{w}_{e_i} + {\beta}_{e_i}\right), \\
u_m^{\mathbf{H}}(\boldsymbol{x})
= \sum_{j=1}^{mq} w_{nk+j}\,
    \xi\!\left(z \boldsymbol{w}_{h_j} + {\beta}_{h_j}\right),
\end{align*}
and the feature mapping $z:\Omega\to\mathbb{R}^{nk}$ is defined by
\begin{equation*}
z(\boldsymbol{x})=
\big[\phi(\boldsymbol{x} \boldsymbol{w}_{e_1} + {\beta}_{e_1}), \dots ,\phi(\boldsymbol{x} \boldsymbol{w}_{e_{nk}} + {\beta}_{e_{nk}})\big].
\end{equation*}

The output weights $\mathbf{W}_\text{out}=[w_1,\dots,w_{nk+mq}]$ are obtained analytically.
\end{definition}

\begin{assumption}[Activation Regularity]
\label{assump:activation_regular}
The activations $\phi(\cdot)$ and $\xi(\cdot)$ are non-constant, bounded, and $C^\infty$.
\end{assumption}

\begin{assumption}[Parameter Distribution]
\label{assump:parameter_distribution}
All parameters of the feature and enhancement mappings are sampled i.i.d.\ from continuous distributions with compact support (or bounded support).
\end{assumption}

\begin{lemma}[Smooth Embedding of Feature Mapping]
\label{lemma:feature_embedding}
Under Assumptions~\ref{assump:activation_regular}--\ref{assump:parameter_distribution},
the random mapping
\begin{equation*}
\boldsymbol{x} \longmapsto z(\boldsymbol{x})
\end{equation*}
is almost surely a smooth embedding (diffeomorphism onto its image) for sufficiently large $n$.
\end{lemma}

\begin{proof}
We prove that the random mapping \(\boldsymbol{x} \mapsto z(\boldsymbol{x})\) is almost surely a smooth embedding for sufficiently large \(n\) by establishing that it is an immersion and injective.

\noindent\textbf{Immersion.}
The Jacobian of \(z\) at \(\boldsymbol{x} \in \Omega\) is the \(nk \times D\) matrix
\begin{equation*}
J_z(\boldsymbol{x}) = \begin{bmatrix}
\phi'(\boldsymbol{x} \boldsymbol{w}_{e_1}+\beta_{e_1}) \boldsymbol{w}_{e_1}^\top \\
\vdots \\
\phi'(\boldsymbol{x} \boldsymbol{w}_{e_{nk}}+\beta_{e_{nk}}) \boldsymbol{w}_{e_{nk}}^\top
\end{bmatrix}.
\end{equation*}
The mapping \(z\) is an immersion if \(J_z(\boldsymbol{x})\) has rank \(D\) for every \(\boldsymbol{x} \in \Omega \), equivalently if the \(D \times D \) matrix
\begin{equation*}
G_n(\boldsymbol{x}) = J_z(\boldsymbol{x})^\top J_z(\boldsymbol{x}) = \sum_{i=1}^{nk} \bigl[\phi'(\boldsymbol{x} \boldsymbol{w}_{e_i}+\beta_{e_i})\bigr]^2\, \boldsymbol{w}_{e_i} \boldsymbol{w}_{e_i}^\top
\end{equation*}
is positive definite. Consider the normalized matrix \(\overline{G}_n(\boldsymbol{x})=\frac{1}{nk}G_n(\boldsymbol{x})\) and its population counterpart
\begin{equation*}
A(\boldsymbol{x})=\mathbb{E}\bigl[\,\bigl(\phi'(\boldsymbol{x}W+\beta)\bigr)^2 \, W W^\top \, \bigr],
\end{equation*}
where \((W,\beta)\) follows the same distribution as \((\boldsymbol{w}_{e_i},{\beta}_{e_i})\). Because \(\phi\) is non‑constant and smooth, and \((W,\beta)\) is drawn from a continuous distribution with finite second moments, \(A(\boldsymbol{x})\) is continuous in \(\boldsymbol{x}\) and positive definite for every \(\boldsymbol{x}\in\Omega\). By compactness of \(\Omega\), there exists \(\lambda>0\) such that \(\lambda_{\min}\bigl(A(\boldsymbol{x})\bigr)\ge\lambda\) for all \(\boldsymbol{x}\in\Omega\).

Take an \(\varepsilon\)-net \(\mathcal{N}_\Omega\) of \(\Omega\) with cardinality \(|\mathcal{N}_\Omega|=O(\varepsilon^{-D})\). For each fixed \(\boldsymbol{x}_0\in\mathcal{N}_\Omega\), the terms \(\bigl(\phi'(\boldsymbol{x}_0 \boldsymbol{w}_{e_i} + {\beta}_{e_i})\bigr)^2 \boldsymbol{w}_{e_i} \boldsymbol{w}_{e_i}^\top \) are i.i.d. and bounded (since \(\phi'\) is bounded and \(\mathbb{E}\|W\|^2<\infty\)). By Hoeffding's inequality and a union bound,
\begin{equation*}
\mathbb{P}\Bigl(\sup_{\boldsymbol{x}_0\in\mathcal{N}_\Omega}\|\overline{G}_n(\boldsymbol{x}_0)-A(\boldsymbol{x}_0)\|>\delta\Bigr)
\le C_1|\mathcal{N}_\Omega|\mathrm{e}^{-C_2n\delta^2}
\end{equation*}
for constants \(C_1,C_2>0\). Choosing \(\delta<\lambda/2\) and using Borel–Cantelli, we obtain almost sure uniform convergence on \(\mathcal{N}_\Omega\). By the smoothness of \(\phi\) and the continuity of \(A\), this convergence extends uniformly to all \(\boldsymbol{x}\in\Omega\). Hence, for all sufficiently large \(n\),
\begin{equation*}
\lambda_{\min}\bigl(\overline{G}_n(\boldsymbol{x})\bigr)\ge\frac{\lambda}{2}>0\qquad\text{almost surely for all }\boldsymbol{x}\in\Omega,
\end{equation*}
implying \(G_n(\boldsymbol{x})\) is positive definite and \(z\) is an immersion.

\noindent\textbf{Injectivity.}
Define the population kernel
\begin{equation*}
K(\boldsymbol{x},\boldsymbol{y})=\mathbb{E}\bigl[\phi(\boldsymbol{x}W+\beta)\,\phi(\boldsymbol{y}W+\beta)\bigr]
\end{equation*}
and the empirical kernel
\begin{equation*}
K_n(\boldsymbol{x},\boldsymbol{y})=\frac{1}{nk}\sum_{i=1}^{nk}\phi(\boldsymbol{x} \boldsymbol{w}_{e_i}+ \beta_{e_i})\,\phi(\boldsymbol{y} \boldsymbol{w}_{e_i}+\beta_{e_i}).
\end{equation*}
The squared distance between \(z(\boldsymbol{x})\) and \(z(\boldsymbol{y})\) satisfies
\begin{equation*}
\frac{1}{nk}\|z(\boldsymbol{x})-z(\boldsymbol{y})\|^2=K_n(\boldsymbol{x},\boldsymbol{x})-2K_n(\boldsymbol{x},\boldsymbol{y})+K_n(\boldsymbol{y},\boldsymbol{y}).
\end{equation*}
Because \(\phi\) is non‑constant and the distribution of \((W,\beta)\) has full support, the kernel \(K\) is characteristic; consequently,
\begin{equation*}
d(\boldsymbol{x},\boldsymbol{y}):=K(\boldsymbol{x},\boldsymbol{x})-2K(\boldsymbol{x},\boldsymbol{y})+K(\boldsymbol{y},\boldsymbol{y})>0\qquad\text{for all }\boldsymbol{x}\neq \boldsymbol{y}.
\end{equation*}
By compactness of \(\Omega\), the function \(d(\boldsymbol{x},\boldsymbol{y})\) is uniformly continuous and attains a positive minimum on \(\{(\boldsymbol{x},\boldsymbol{y})\in\Omega\times\Omega:\|\boldsymbol{x}-\boldsymbol{y}\|\ge\delta\}\) for any fixed \(\delta>0\).

Again using an \(\varepsilon\)-net argument, one shows that \(K_n\) converges uniformly to \(K\) on \(\Omega\times\Omega\). Hence, for sufficiently large \(n\),
\begin{equation*}
\bigl|\,\tfrac{1}{nk}\|z(\boldsymbol{x})-z(\boldsymbol{y})\|^2-d(\boldsymbol{x},\boldsymbol{y})\bigr|<\tfrac{1}{2}\inf_{\|\boldsymbol{x}'-\boldsymbol{y}'\|\ge\delta}d(\boldsymbol{x}',\boldsymbol{y}')
\end{equation*}
almost surely for all \(\boldsymbol{x},\boldsymbol{y}\in\Omega\). This guarantees \(\|z(\boldsymbol{x})-z(\boldsymbol{y})\|>0\) whenever \(\|\boldsymbol{x}-\boldsymbol{y}\|\ge\delta\).

For points with \(\|\boldsymbol{x}-\boldsymbol{y}\|<\delta\), the local injectivity follows from the fact that \(z\) is an immersion: by the inverse function theorem, each \(\boldsymbol{x} \in \Omega \) possesses a neighbourhood \(U_{\boldsymbol{x}}\) on which \(z\) is a diffeomorphism onto its image, hence injective. Compactness of \(\Omega\) yields a finite subcover \(\{U_{\boldsymbol{x}_k}\}_{k=1}^K\); let \(\delta'\) be a Lebesgue number of this cover. Choosing \(\delta\le\delta'\), any two distinct points with \(\|\boldsymbol{x}-\boldsymbol{y}\|<\delta\) lie in the same \(U_{\boldsymbol{x}_k}\) and therefore satisfy \(z(\boldsymbol{x})\neq z(\boldsymbol{y})\). Thus \(z\) is injective.

\noindent\textbf{Embedding.}
Since \(\Omega\) is compact and \(z:\Omega\to\mathbb{R}^{nk}\) is a smooth injective immersion, it is a smooth embedding (diffeomorphism onto its image). This completes the proof.

\end{proof}

\begin{table*}[!ht]
    \centering
    \caption{\textbf{Summary of benchmark problems and model hyperparameters.} The table lists the test cases (TC-1 to TC-11), their governing equations, domains, and boundary and initial conditions. $N_f$, $N_\text{bc}$, and $N_\text{ic}$ denote the number of interior collocation points, boundary points, and initial points, respectively. `Total Params' refers to the number of trainable weights in the PIBLS model.}
    \label{tab:benchmark_setup}
    \small
    \renewcommand{\arraystretch}{1.5} 
    \begin{tabular}{@{}c l >{\raggedright\arraybackslash}p{6.5cm} c c l@{}}
        
        \toprule
        & \textbf{ID} & \textbf{Description \newline (PDE, Domain, IC/BC, Solution)} & \textbf{Dim.} & \textbf{Total Params} & \textbf{Points \newline ($N_f, N_\text{bc}, N_\text{ic}$)} \\
        \midrule
        
        \parbox[c]{3mm}{\multirow{20}{*}{\rotatebox[origin=c]{90}{Linear}}}
        & TC-1 & Advection: $u_x = f(x)$ \newline Dom: $(0, 1)$, BCs: Dirichlet. \newline Sol: $u(x) = \sin(2\pi x)\cos(4\pi x) + 1$ & 1D & 1240 & (900, 2, 0) \\
        
        & TC-2 & Diffusion: $u_{xx} = f(x)$ \newline Dom: $(0, 1)$, BCs: Dirichlet. \newline Sol: $u(x) = \sin(\pi x/2)\cos(2\pi x) + 1$ & 1D & 140 & (100, 2, 0) \\
        
        & TC-3 & Adv.-Diff.: $u_x - 0.2 u_{xx} = 0$ \newline Dom: $(0, 1)$, BCs: Dirichlet. \newline Sol: $u(x) = (e^{5x} - 1) / (e^5 - 1)$ & 1D & 720 & (300, 2, 0) \\
        
        & TC-4 & Advection: $u_x + 0.5 u_y = f(x,y)$ \newline Dom: $(-1, 1)^2$, BCs: Dirichlet. \newline Sol: $u(x,y) = \frac{1}{2}\cos(\pi x)\sin(\pi y)$ & 2D & 1606 & (2800, 700, 0) \\
        
        & TC-5 & Poisson: $u_{xx} + u_{yy} = f(x,y)$ \newline Dom: $(0, 1)^2$, BCs: Dirichlet. \newline Sol: $u(x,y) = \frac{1}{2} + e^{-2x^2 - 4y^2}$ & 2D & 1400 & (1900, 400, 0) \\
        
        & TC-6 & Poisson: $u_{xx} + u_{yy} = f(x,y)$ \newline Dom: $(0, 1)^2$, BCs: Dirichlet. \newline Sol: $u(x,y) = \frac{1}{2} + e^{-(x-0.6)^2 - (y-0.6)^2}$ & 2D & 1240 & (2500, 1900, 0) \\
        
        & TC-7 & Wave (Const.): $u_t + u_x = 0$ \newline Dom: $(-1, 1) \times (0, 0.5)$, BCs: Periodic. \newline IC: $u(x,0) = \sin(\pi x)$ \newline Sol: $u(x,t) = \sin(\pi(x-t))$ & 1D+time & 1300 & (3800, 1700, 2300) \\
        
        & TC-8 & Wave (Var.): $u_t + (1+x)u_x = 0$ \newline Dom: $(-1, 1) \times (0, 0.5)$, BCs: Dirichlet. \newline IC: $u(x,0) = \sin(\pi x)$ \newline Sol: $u(x,t) = \sin(\pi((1+x)e^{-t} - 1))$ & 1D+time & 1300 & (1900, 800, 200) \\
        
        \midrule
        
        \parbox[c]{3mm}{\multirow{7.5}{*}{\rotatebox[origin=c]{90}{Nonlinear}}}
        & TC-9 & Helmholtz: $u_{xx} - 50u + 10\sin(u) = f(x)$ \newline Dom: $(0, 8)$, BCs: Dirichlet. \newline Sol:  \newline$u(x) = \sin(3\pi x + \frac{3\pi}{20})\cos(4\pi x - \frac{2\pi}{5}) + 1.5 + \frac{x}{10}$ & 1D & 1400 & (2800, 1200, 0) \\
        
        & TC-10 & Spring: $u_{tt} + 4u + 0.1\sin(u) = f(t)$ \newline Dom: $t \in (0, 2.5)$, ICs: $u(0)=0, u'(0)=0$. \newline Sol: $u(t) = t \sin(t)$ & time & 1100 & (1400, 0, 1) \\
        
        & TC-11 & Burger's: $u_t + uu_x - 0.01u_{xx} = f(x,t)$ \newline 
        Dom: $(0, 1) \times (0, 0.25)$, BCs: Dirichlet. \newline 
        Sol: $u(x,t) = (1 + \frac{x}{10})(1 + \frac{t}{10}) \times$ \newline 
        \phantom{Sol: }$\left[2\cos\left(\pi x + \frac{2\pi}{5}\right) + \frac{3}{2}\cos\left(2\pi x - \frac{3\pi}{5}\right)\right] \times$ \newline 
        \phantom{Sol: }$\left[2\cos\left(\pi t + \frac{2\pi}{5}\right) + \frac{3}{2}\cos\left(2\pi t - \frac{3\pi}{5}\right)\right]$ 
        & 1D+time & 415 & (1300, 800, 800) \\
        
        \bottomrule
    \end{tabular}
\end{table*}

\begin{theorem}[Universal Approximation in Sobolev Spaces]
\label{thm:universal_approximation}
For any compact domain with Lipschitz boundary $\Omega \subset \mathbb{R}^D$ and any continuous function $
u^* \in H^s(\Omega), s\in\{1,2\}$, there exists a sequence of PIBLS approximants $\{\hat{u}_{n,m}\}_{n,m\in\mathbb{N}}$ under Assumptions~\ref{assump:activation_regular}--\ref{assump:parameter_distribution} such that
\begin{equation*}
\lim_{n,m\to\infty}
\|u^* - \hat{u}_{n,m}\|_{H^s(\Omega)} = 0.
\end{equation*}

\end{theorem}
\begin{proof}
Fix $u^*\in H^s(\Omega)$ and $\epsilon>0$.  
Let $n\in\mathbb{N}$ be large enough that the feature map $\boldsymbol{x} \mapsto z(\boldsymbol{x})$ is a smooth
embedding onto $z(\Omega)$ (Lemma~\ref{lemma:feature_embedding}).  
For this $n$ define the feature-layer approximation
\begin{equation*}
u^{\mathbf{Z}}_n(\boldsymbol{x})=\sum_{i=1}^{nk} w_i\,\phi(\boldsymbol{x} \boldsymbol{w}_{e_i} + {\beta}_{e_i})
\end{equation*}
and the residual
\begin{equation*}
r_n(\boldsymbol{x})=u^*(\boldsymbol{x})-u^{\mathbf{Z}}_n(\boldsymbol{x}).
\end{equation*}

Since $u^*\in H^s(\Omega)$ and $u^{\mathbf{Z}}_n$ is smooth on $\Omega$, we have
$r_n\in H^s(\Omega)$. In particular $r_n$ belongs to $L^2(\Omega)$ (hence is
square-integrable) and is bounded in the $H^s$-norm, and thus is approximable by smooth functions in the $H^s$-topology.

By density of $C^\infty(\Omega)$ in $H^s(\Omega)$ there exists
$h\in C^\infty(\Omega)$, such that for any $\epsilon > 0$~\cite{sobolev_spaces}, we have
\begin{equation*}
\|r_n - h\|_{H^s(\Omega)} < \frac{\epsilon}{2}.
\end{equation*}

Because $\boldsymbol{x} \mapsto z(\boldsymbol{x})$ is a diffeomorphism from $\Omega$ onto the compact
set $z(\Omega)$, the pullback $\tilde h := h\circ z^{-1}$ is in
$C^\infty\big(z(\Omega)\big)$. Standard results on Sobolev spaces under smooth
coordinate changes (equivalence of Sobolev norms on compact sets under a
diffeomorphism) yield a constant $C_z\ge 1$ (depending only on $z$ and
$\Omega$) such that for any function $v$ on $z(\Omega)$,
\begin{equation*}
C_z^{-1}\|v\circ z\|_{H^s(\Omega)} \le \|v\|_{H^s(z(\Omega))} \le C_z\|v\circ z\|_{H^s(\Omega)}.
\end{equation*}

Consider enhancement-layer approximants of the form
\begin{equation*}
u_m^{\mathbf{H}}(\boldsymbol{x})=\sum_{j=1}^{mq} w_{nk+j}\,\xi\!\big(z\boldsymbol{w}_{h_j}+{\beta}_{h_j}\big),
\end{equation*}
which, viewed on $z(\Omega)$, are finite linear combinations of the activation
$\xi$ applied to affine functionals on $\mathbb{R}^{nk}$. 

By the  universal approximation property in Theorem~3.9 of~\cite{de2025approximation}, for bounded, nonconstant, $C^\infty$ activations in
Sobolev norms on compact domains (applied to the embedded domain
$z(\Omega)$ and under the i.i.d.\ parameter sampling of
Assumption~\ref{assump:parameter_distribution}), there exists $m\in\mathbb{N}$
and output weights $w_{nk+1},\dots,w_{nk+mq}$ such that for any $\epsilon > 0$, the corresponding enhancement approximation satisfies
\begin{equation*}
\big\|u^{\mathbf{H}}_m\circ z^{-1} - \tilde h\big\|_{H^s(z(\Omega))} < C_z^{-1}\,\frac{\epsilon}{2}.
\end{equation*}

Using the norm equivalence above and pulling back to $\Omega$ we obtain
\begin{equation*}
\|u^{\mathbf{H}}_m - h\|_{H^s(\Omega)} < \frac{\epsilon}{2}.
\end{equation*}

Combining the two approximations gives
\begin{align*}
\|u^* - (u^{\mathbf{Z}}_n + u^{\mathbf{H}}_m)\|_{H^s(\Omega)} 
    &= \|r_n - u^{\mathbf{H}}_m\|_{H^s(\Omega)} \\
    &\le \|r_n - h\|_{H^s(\Omega)} + \|h - u^{\mathbf{H}}_m\|_{H^s(\Omega)} \\
    &< \frac{\epsilon}{2} + \frac{\epsilon}{2} \\
    &= \epsilon.
\end{align*}
Hence, we could conclude that
\begin{equation*}
\lim_{n,m\to\infty}
\|u^* - \hat{u}_{n,m}\|_{H^s(\Omega)} = 0.
\end{equation*}
\end{proof}

\begin{table*}[!ht]
    \centering
    \caption{\textbf{Performance comparison on linear test cases.} Values in each cell: Max Error / $L_2$ Error / Training Time (seconds). PINN \#1 has a total parameter count comparable to PIBLS; PINN \#2 has a fixed architecture (10 hidden layers, 50 nodes/layer). FEM results correspond to a mesh with 20,000 elements (1D) or 150$\times$150 (2D). The best accuracy for each test case is highlighted in red. Full experimental details are provided in the Methods.}
    \label{tab:linear_results}
    \footnotesize
    \setlength{\tabcolsep}{2pt}
    \renewcommand{\arraystretch}{1.2}
    \begin{tabular}{lcccc}
        \toprule
        Method & TC-1 & TC-2 & TC-3 & TC-4 \\
        \midrule
        PINN \#1 & 
        2.16e-04 / 9.18e-05 / 16.61 & 
        8.89e-05 / 6.73e-05 / 13.77 & 
        1.95e-04 / 9.46e-05 / 16.75 & 
        1.62e-02 / 5.16e-03 / 29.71 \\
        \midrule
        PINN \#2 & 
        1.43e-04 / 4.26e-05 / 325 & 
        2.44e-04 / 1.34e-04 / 310 & 
        6.27e-06 / 3.17e-06 / 382 & 
        9.78e-04 / 2.31e-04 / 849 \\
        \midrule
        PIELM & 
        5.96e-07 / 1.61e-07 / 0.10 & 
        3.57e-07 / 9.67e-08 / 0.02 & 
        1.29e-07 / 3.69e-08 / 0.03 & 
        8.71e-08 / 2.43e-08 / 0.38 \\
        \midrule
        FEM & 
        1.78e-07 / 1.27e-07 / 0.03 &    
        2.59e-08 / 2.00e-08 / 0.03 & 
        1.25e-09 / 7.19e-10 / 0.03 & 
        2.77e-08 / 1.39e-08 / 0.05 \\
        \midrule
        PIBLS & 
        \textcolor{red}{2.00e-15} / \textcolor{red}{5.99e-16} / 0.13 & 
        \textcolor{red}{4.22e-15} / \textcolor{red}{1.05e-15} / 0.01 & 
        \textcolor{red}{7.77e-16} / \textcolor{red}{2.55e-16} / 0.03 & 
        \textcolor{red}{8.40e-13} / \textcolor{red}{2.98e-13} / 0.76 \\
        \bottomrule
    \end{tabular}
    \vspace{6pt} 
    
    \begin{tabular}{lcccc}
        \toprule
        Method & TC-5 & TC-6 & TC-7 & TC-8 \\
        \midrule
        PINN \#1 & 
        2.02e-02 / 6.14e-03 / 58.01 & 
        2.69e-03 / 5.90e-04 / 60.1 & 
        2.54e-02 / 6.04e-03 / 37.34 & 
        1.47e-03 / 7.57e-04 / 23.77 \\
        \midrule
        PINN \#2 & 
        4.06e-03 / 3.72e-01 / 1663 & 
        1.84e-03 / 4.46e-04 / 1844 & 
        4.45e-03 / 5.24e-04 / 698 & 
        3.71e-03 / 5.35e-04 / 1116 \\
        \midrule
        PIELM & 
        1.06e-07 / 2.74e-07 / 0.24 & 
        1.41e-07 / 4.81e-08 / 0.28 & 
        5.98e-05 / 2.21e-06 / 0.57 & 
        4.88e-07 / 9.73e-08 / 0.25 \\
        \midrule
        FEM & 
        8.97e-10 / 3.31e-10 / 0.05 & 
        1.23e-10 / 6.43e-11 / 0.05 & 
        5.08e-05 / 4.22e-02 / 0.03 & 
        1.32e-04 / 2.66e-05 / 0.01 \\
        \midrule
        PIBLS & 
        \textcolor{red}{6.51e-14} / \textcolor{red}{7.69e-15} / 0.38 & 
        \textcolor{red}{3.09e-14} / \textcolor{red}{3.37e-15} / 0.51 & 
        \textcolor{red}{9.49e-13} / \textcolor{red}{2.39e-13} / 0.70 & 
        \textcolor{red}{2.09e-11} / \textcolor{red}{4.71e-12} / 0.24 \\
        \bottomrule
    \end{tabular}
\end{table*}

\section{Results}

We conducted a comprehensive evaluation across a diverse suite of 11 benchmark problems, summarized in Table~\ref{tab:benchmark_setup}. The evaluation begins by assessing accuracy and computational speed on eight linear PDEs, providing direct comparisons against standard PINN, PIELM, and FEM. Subsequently, the scope extends to challenging nonlinear problems, employing locELM~\cite{LocalELM} in place of PIELM for its capacity to solve nonlinear equations. The evaluation concludes with a sensitivity analysis of key hyperparameters to assess the robustness of the proposed method.

\subsection{Experimental setup and evaluation}
For all theoretical test cases, the source terms as well as the boundary and initial conditions were derived from manufactured solutions, enabling the precise calculation of solution errors. The detailed definitions of these test cases, including the physical models and computational constraints, are outlined in Table~\ref{tab:benchmark_setup}. We then compared the performance of PIBLS against five key baseline models. These baselines include:
\begin{itemize}
\item \textbf{PINN \#1:} A standard PINN whose architecture was tuned to match the total parameter count of the corresponding PIBLS model. The model was trained using the Adam optimizer for 5000 iterations.
\item \textbf{PINN \#2:} A deeper PINN with a fixed architecture consisting of 10 hidden layers, each with 50 nodes. This model was trained with Adam for 45,000 iterations, followed by L-BFGS optimization.
\item \textbf{PIELM:} Physics-Informed Extreme Learning Machines (PIELM). To ensure a fair comparison, the number of random nodes was set to match the total number of trainable parameters in the corresponding PIBLS model.
\item \textbf{locELM:} A method that combines domain decomposition with Extreme Learning Machines. This method served as the primary baseline for the nonlinear benchmarks (TC-9, TC-10, and TC-11), and its NLSQ-perturb parameters were configured as described in the original work~\cite{LocalELM}.
\item \textbf{FEM:} The standard Finite Element Method, implemented in FEniCS. A high-resolution mesh with 20,000 elements was used for 1D problems, and a $150 \times 150$ mesh was used for 2D problems.
\end{itemize}
Our PIBLS model was implemented in Python, and all computational benchmarks were conducted on a system equipped with an Intel Core i5-12600KF CPU and 32GB of RAM, running on the Windows Subsystem for Linux (WSL) platform.

\subsection{Performance on linear PDEs}
The evaluation begins with a range of linear PDEs taken from~\cite{PINN_complex_geometry,TC8_equations}, comprising both one-dimensional (1D) and two-dimensional (2D) test cases.

\subsubsection{One-dimensional steady-state problems}

The initial benchmarks involved three 1D steady-state PDEs: the advection equation (TC-1), the diffusion equation (TC-2), and the advection-diffusion equation (TC-3). For these benchmarks, we used PIBLS architectures with total trainable parameters ranging from 140 to 1240 and limited sampling points, as detailed in Table~\ref{tab:benchmark_setup}.

As shown in Fig.~\ref{fig:tc1-tc3}a-c, the PIBLS predictions match the exact analytical solutions remarkably closely. The corresponding point-wise error plots in Fig.~\ref{fig:tc1-tc3}d-f further confirm that errors remain negligible, approaching machine precision across the entire domain. By avoiding iterative optimization, PIBLS is not susceptible to the local minima or optimization pathologies that can affect gradient-based solvers. Instead, it finds the globally optimal solution for the given basis functions in a single step.

Quantitative comparisons in Table~\ref{tab:linear_results} show the full extent of this precision. For TC-1, PIBLS achieved an $L_2$ error of $5.99 \times 10^{-16}$. This precision is particularly remarkable as it is 10 to 11 orders of magnitude more accurate than both the parameter-matched PINN \#1 ($9.18 \times 10^{-5}$) and the deep PINN \#2 ($4.26 \times 10^{-5}$), and 9 orders lower than PIELM ($1.61 \times 10^{-7}$). This comparison highlights the practical limitations of gradient-based PINN solvers, which struggled to converge and stagnated at significantly higher error rates. In contrast, the PIBLS direct-solve method reliably converged to the limit of machine precision. This trend of machine-precision accuracy held for the other 1D test cases where the PIBLS error of $1.05 \times 10^{-15}$ for TC-2 remained 7 to 10 orders lower than all baselines. Similarly for TC-3, PIBLS outperformed all competitors by 8 to 10 orders of magnitude with an error of $2.55 \times 10^{-16}$. Notably, in all three 1D test cases, the PIBLS solution was significantly more accurate than even the high-resolution FEM and PIELM baselines.

Critically, this precision was achieved with exceptional speed. For these 1D benchmarks, PIBLS solution times ranged from 0.01~s to 0.13~s. In contrast, the standard PINN \#1 required between 13.77 and 16.75~s to converge, while the deep PINN \#2 demanded significantly more time, ranging from 310 to 382~s. This distinct performance gap underscores a fundamental difference in computational complexity, as the framework replaces the computationally expensive iterative optimization process with the cost of a single least-squares solve.

\subsubsection{Two-dimensional steady-state problems}

\begin{figure}[!t]
    \centering
    \includegraphics[width=0.95\textwidth]{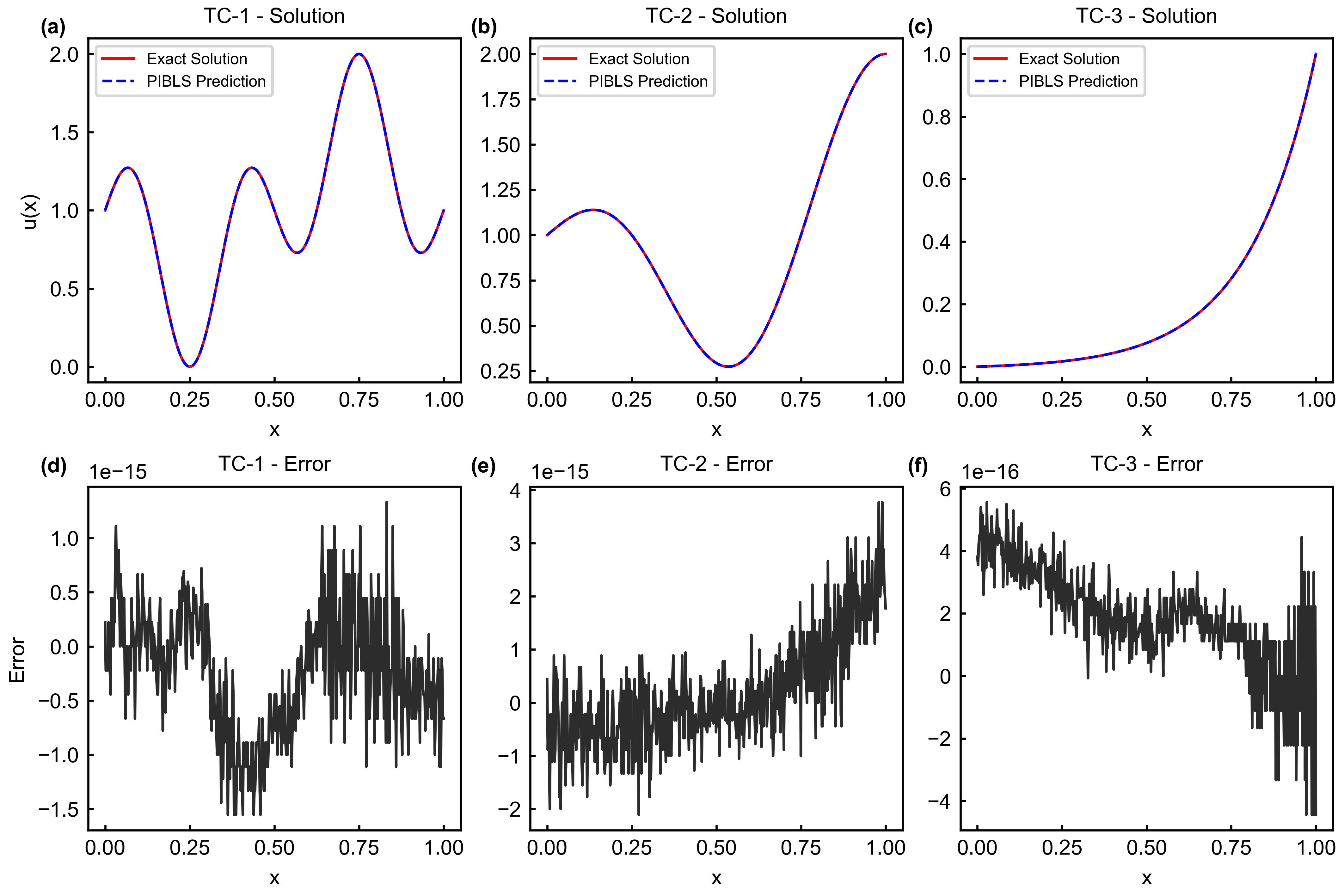}
    \caption{\textbf{Solutions and errors for 1D steady-state problems.} \textbf{a}-\textbf{c} Comparison of PIBLS predictions (blue dashed lines) and exact solutions (red solid lines) for TC-1, TC-2, and TC-3, respectively. \textbf{d}-\textbf{f} Corresponding point-wise error plots.}
    \label{fig:tc1-tc3}
\end{figure}

\begin{figure}[!t]
    \centering
    \includegraphics[width=0.95\textwidth]{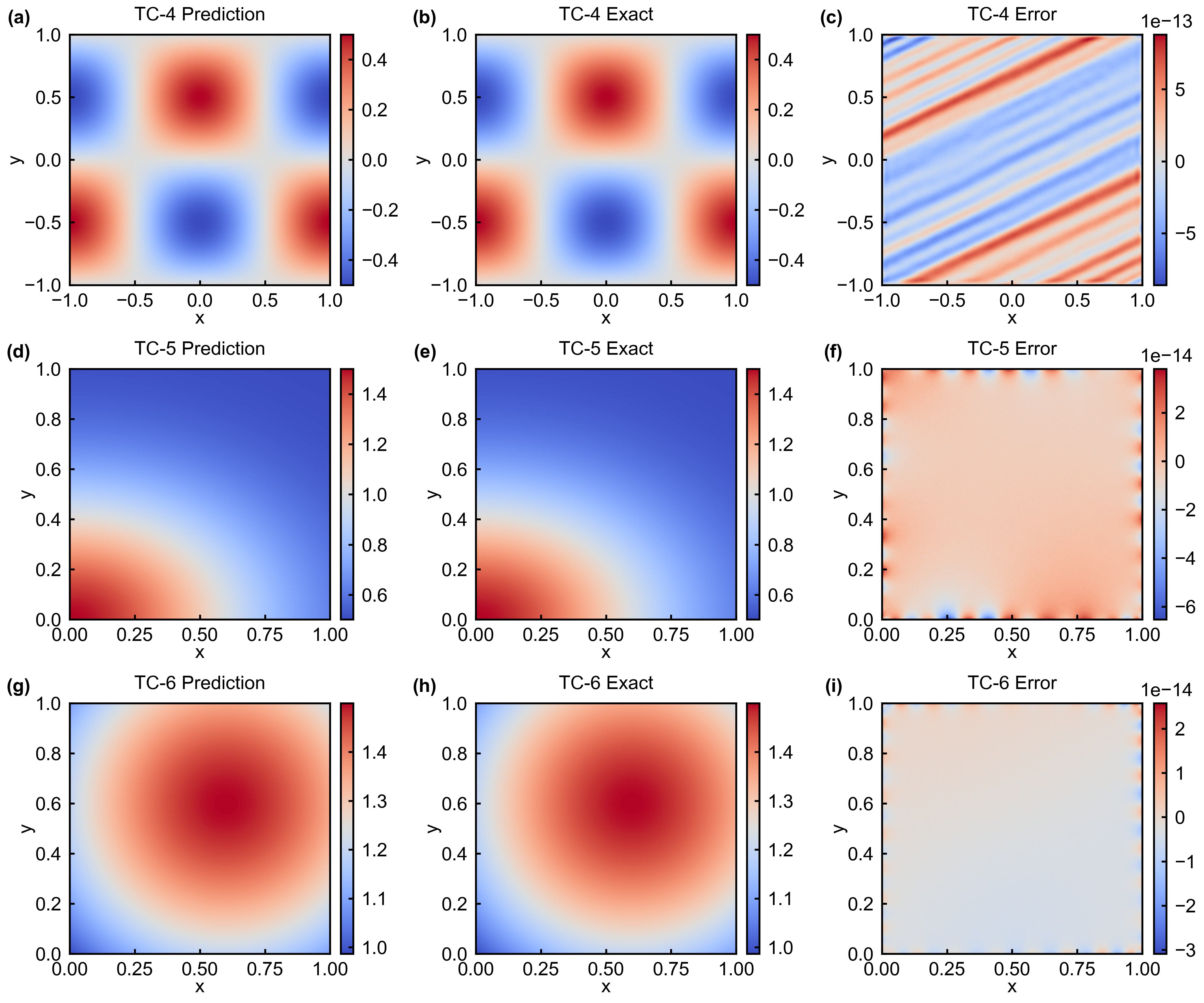}
    \caption{\textbf{Solutions and error contours for 2D linear steady-state problems.} \textbf{a}, \textbf{d}, \textbf{g} PIBLS predictions. \textbf{b}, \textbf{e}, \textbf{h} Exact solutions. \textbf{c}, \textbf{f}, \textbf{i} Point-wise error contours. The rows correspond to test cases TC-4  (\textbf{a}-\textbf{c}), TC-5 (\textbf{d}-\textbf{f}), and TC-6 (\textbf{g}-\textbf{i}).}
    \label{fig:tc4-tc6}
\end{figure}

To assess the framework's ability to handle multi-dimensional inputs and more complex solution field distributions, we next tested PIBLS on 2D steady-state problems: the advection equation (TC-4) and two Poisson equations (TC-5, TC-6). For this set of 2D benchmarks, we utilized PIBLS architectures with total parameters ranging from 1240 to 1606, with specific sampling point distributions detailed in Table~\ref{tab:benchmark_setup}.

As illustrated in Fig.~\ref{fig:tc4-tc6}, the PIBLS solutions are qualitatively indistinguishable from the exact solutions, and the corresponding error contours confirm that the point-wise error remains vanishingly small across the 2D domains.

As detailed in Table~\ref{tab:linear_results}, PIBLS substantially outperformed all baselines in accuracy. 
For TC-4, its $L_2$ error of $2.98 \times 10^{-13}$ was 5 orders of magnitude lower than the high-resolution FEM ($1.39 \times 10^{-8}$). This precision contrasts sharply with the gradient-based solvers, as both PINN \#1 ($5.16 \times 10^{-3}$) and the deep PINN \#2 ($2.31 \times 10^{-4}$) struggled to converge to a comparable solution. Similarly for TC-5 and TC-6, PIBLS errors reached $7.69 \times 10^{-15}$ and $3.37 \times 10^{-15}$, respectively. This represents an improvement of 4 to 5 orders of magnitude over the fine-mesh FEM ($3.31 \times 10^{-10}$ and $6.43 \times 10^{-11}$, respectively) and an even greater improvement over both PINN and PIELM models. The ability of PIBLS to outperform high-resolution FEM is particularly notable, achieving near-analytical fidelity that typically necessitates computationally prohibitive fine-mesh discretizations.

Computationally, this efficiency was confirmed by the runtime analysis. PIBLS remained highly efficient, with solve times for the 2D problems (0.38–0.76~s) consistently comparable to the other fast, non-iterative solvers like PIELM and FEM. The critical distinction emerged against the gradient-based architectures. For TC-4, PIBLS was over 39 times faster than PINN \#1 (29.71~s) and over 1100 times faster than the deep PINN \#2 (849.74~s). This speed advantage was even more stark for TC-5 and TC-6, where PIBLS solved in under a second, while the deep PINN \#2 required 1663~s and 1844~s, respectively. This massive speed differential underscores the severe computational burden of training deep, iterative-gradient architectures, a bottleneck that PIBLS effectively eliminates.

\subsubsection{One-dimensional time-dependent problems}

\begin{figure}[!t]
    \centering
    \includegraphics[width=0.95\textwidth]{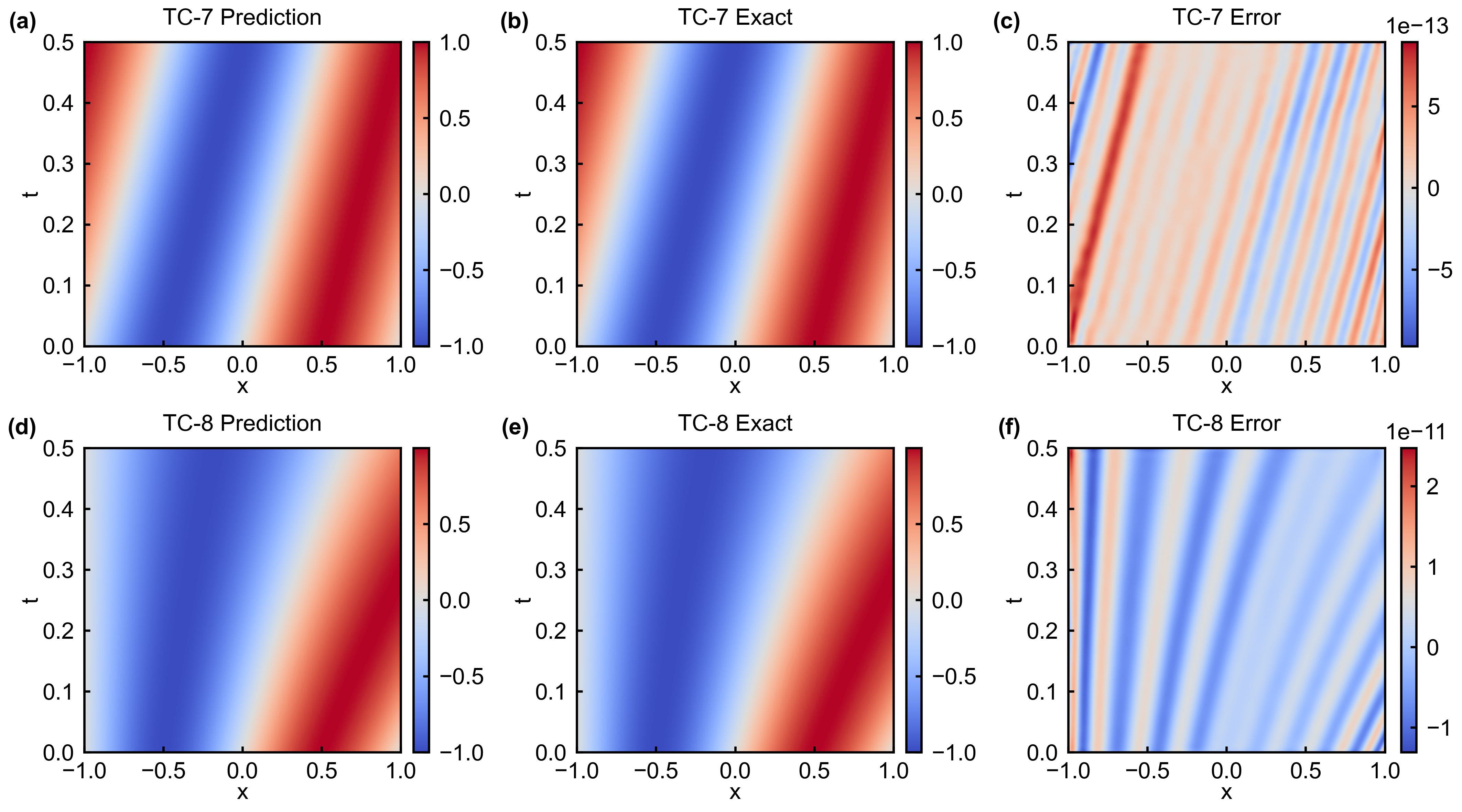}
    \caption{\textbf{Spatio-temporal solutions and errors contours for 1D time-dependent advection problems.} \textbf{a}-\textbf{c} PIBLS predicted solution surface (\textbf{a}), exact solution surface (\textbf{b}), and point-wise error contour (\textbf{c}) for TC-7. \textbf{d}-\textbf{f} PIBLS predicted solution surface (\textbf{d}), exact solution surface (\textbf{e}), and point-wise error contour (\textbf{f}) for TC-8.}
    \label{fig:tc7-tc8_results}
\end{figure}

We further evaluated PIBLS on 1D time-dependent advection problems to test its ability to capture spatio-temporal dynamics. These included a constant coefficient advection equation with periodic boundary conditions (TC-7) and a variable coefficient advection equation with Dirichlet boundary conditions (TC-8)~\cite{TC8_equations}. For these spatio-temporal problems, time $t$ is treated as an additional input coordinate, allowing PIBLS to learn the entire spatio-temporal solution field at once. As specified in Table~\ref{tab:benchmark_setup}, we used PIBLS models with 1300 trainable parameters for these tasks.

As shown in the solution surfaces in Fig.~\ref{fig:tc7-tc8_results}, PIBLS accurately captures the complex wave propagation dynamics at different time points. This qualitative success is corroborated by the quantitative metrics in Table~\ref{tab:linear_results}. PIBLS achieved $L_2$ errors of $2.39 \times 10^{-13}$ for TC-7 and $4.71 \times 10^{-12}$ for TC-8. 
These results highlight a decisive advantage of the framework. While its solution times (0.70~s and 0.24~s) were on a comparable order of magnitude to PIELM (0.57~s and 0.25~s), the achieved accuracy was 5 to 7 orders of magnitude greater. 
Simultaneously, the framework maintains a substantial speed advantage over the PINN baselines, which required between 23~s and 1116~s to converge. This ability to combine machine precision with extreme speed demonstrates the capability of the broad feature space to represent complex, time-evolving dynamics within a single-step optimization process.

\subsection{Solving challenging nonlinear PDEs}
We subsequently tested PIBLS on a series of nonlinear PDEs from~\cite{LocalELM}. The nonlinearity of the differential operators introduces complex optimization problems that, as our results show, often impede the convergence of certain solvers.

\subsubsection{Nonlinear steady-state and dynamic problems}

\begin{table*}[!ht]
    \centering
    \caption{\textbf{Performance comparison on nonlinear test cases.} Values in each cell: Max Error / $L_2$ Error / Training Time (seconds). PINN \#1, PINN \#2, and FEM definitions are consistent with those in Table~\ref{tab:linear_results}. The best accuracy for each test case is highlighted in red. Note: For TC-10, standard spatial FEM is not applicable; $*$ denotes results obtained using Newmark-beta time integration as the numerical baseline.}
    \label{tab:nonlinear_tableresults}
    \footnotesize
    \setlength{\tabcolsep}{2pt}
    \renewcommand{\arraystretch}{1.2} 
    \begin{tabular}{lccc}
        \toprule
        Method & TC-9 & TC-10 & TC-11 \\
        \midrule
        PINN \#1 & 
        1.31e+00 / 5.44e-01 / 367 & 
        1.05e-04 / 5.38e-05 / 26.5 & 
        1.82e+01 / 4.64e+00 / 75.9 \\ 
        \midrule
        PINN \#2 & 
        1.30e+00 / 5.44e-01 / 1082 & 
        2.53e-06 / 2.85e-06 / 812 & 
        2.95e+01 / 5.74e+00 / 1835 \\ 
        \midrule
        locELM & 
        1.45e-09 / 2.34e-10 / 7.7 & 
        2.82e-11 / 1.12e-11 / 0.34 & 
        1.85e-08 / 4.44e-09 / 27.6 \\ 
        \midrule
        FEM & 
        5.38e-08 / 6.66e-07 / 0.03 &    
        4.41e-10$^*$ / 3.02e-10$^*$ / 3.72$^*$ & 
        4.93e-02 / 2.04e-02 / 0.06 \\ 
        \midrule
        PIBLS & 
        \textcolor{red}{2.76e-11} / \textcolor{red}{4.69e-12} / 6.76 & 
        \textcolor{red}{8.88e-16} / \textcolor{red}{2.40e-16} / 3.2 & 
        \textcolor{red}{8.63e-10} / \textcolor{red}{9.88e-11} / 0.56 \\
        \bottomrule
    \end{tabular}
\end{table*}

\begin{figure}[!t]
    \centering
    \includegraphics[width=0.95\textwidth]{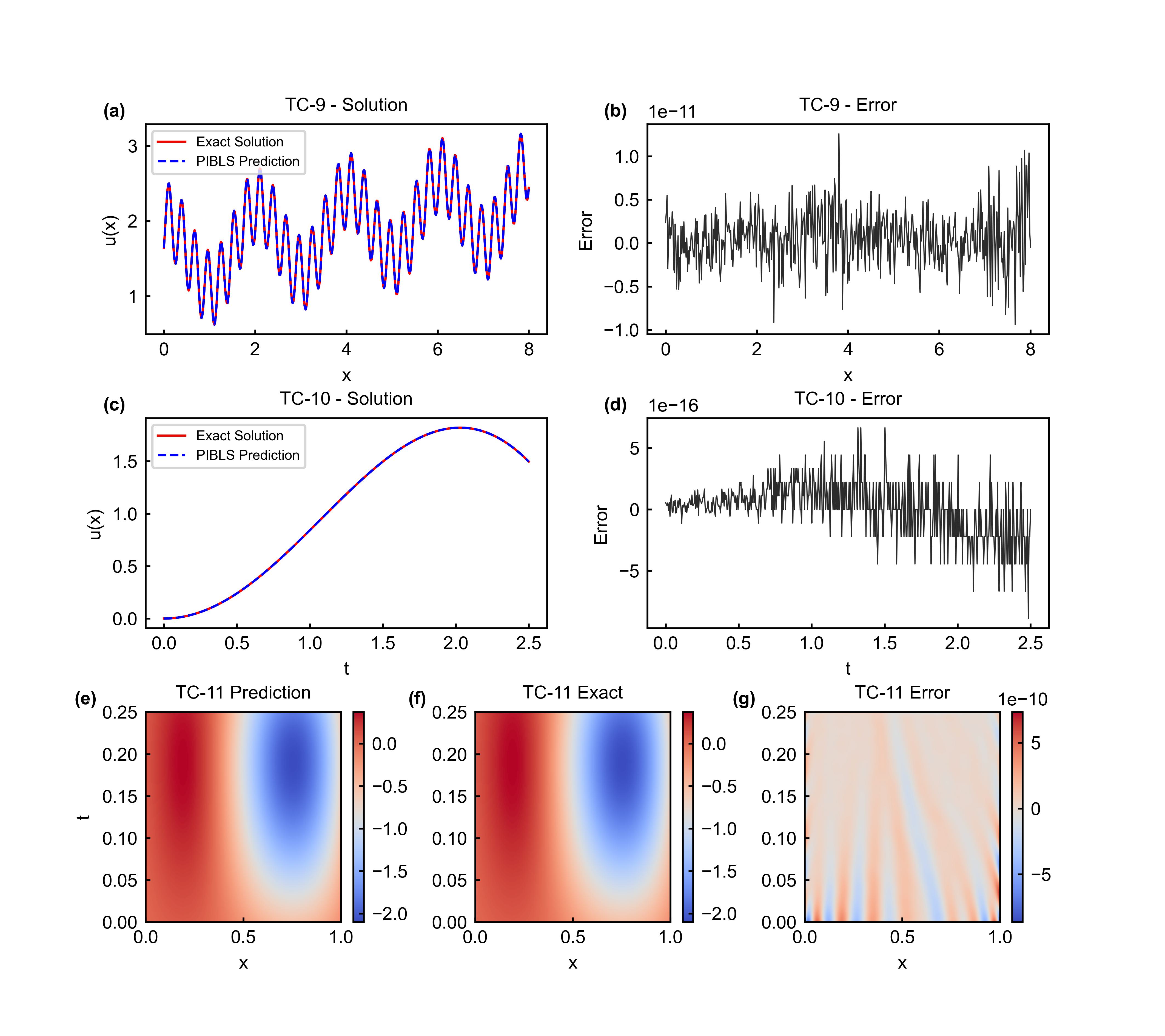}
    \caption{\textbf{Solutions and errors for nonlinear PDEs.} \textbf{a}, \textbf{b} Comparison of PIBLS prediction (blue dashed line) and exact solution (red solid line) (\textbf{a}) and corresponding point-wise error (\textbf{b}) for TC-9. \textbf{c}, \textbf{d} Solution comparison (\textbf{c}) and point-wise error (\textbf{d}) for TC-10. \textbf{e}-\textbf{g} PIBLS predicted solution surface (\textbf{e}), exact solution surface (\textbf{f}), and point-wise error contour (\textbf{g}) for TC-11.}
    \label{fig:nonlinear_figresults}
\end{figure}

We considered three challenging nonlinear benchmarks: the steady-state nonlinear Helmholtz equation (TC-9), the nonlinear spring oscillator equation (TC-10), and the time-dependent viscous Burger's equation (TC-11). For these nonlinear test cases, we utilized PIBLS architectures with total parameters ranging from 415 to 1400, whose specific configurations for each problem are detailed in Table~\ref{tab:benchmark_setup}.

Fig.~\ref{fig:nonlinear_figresults} illustrates the comparative results between PIBLS predictions and exact solutions, demonstrating excellent agreement across all three test cases. The quantitative comparisons in Table~\ref{tab:nonlinear_tableresults} highlight the framework's performance. For TC-9, PIBLS achieved an $L_2$ error of $4.69 \times 10^{-12}$. This result was two orders of magnitude lower than the locELM baseline ($2.34 \times 10^{-10}$), while PINN baselines struggled to converge. This advantage was even more pronounced in TC-10, where PIBLS achieved near-machine accuracy with an $L_2$ error of $2.40 \times 10^{-16}$. This performance surpassed all baselines, including locELM ($1.12 \times 10^{-11}$) and FEM ($3.02 \times 10^{-10}$). Most notably, for the challenging TC-11, PIBLS maintained high accuracy with $L_2$ error of $9.88 \times 10^{-11}$ where PINN models failed completely (errors $>4$) and locELM lagged by two orders of magnitude. This significant accuracy gap highlights the superiority of the PIBLS architecture.

Computationally, PIBLS consistently outperformed the PINN baselines, reducing solution times by orders of magnitude. Compared to locELM, the framework demonstrated superior efficiency on complex dynamics. Notably, for TC-11, PIBLS achieved convergence in just 0.56~s, outpacing locELM (27.6~s) by a factor of approximately 50. While locELM retained a speed advantage for TC-10, PIBLS achieved accuracy improvements of several orders of magnitude while maintaining acceptable computational costs. Crucially, the integration of the expressive BLS basis with the Enhanced NLSQ-perturb solver empowers the framework to resolve challenging nonlinearities with an accuracy that significantly surpasses all other methods in this study.

\begin{figure}[t]
   \centering
   \includegraphics[width=0.95\textwidth]{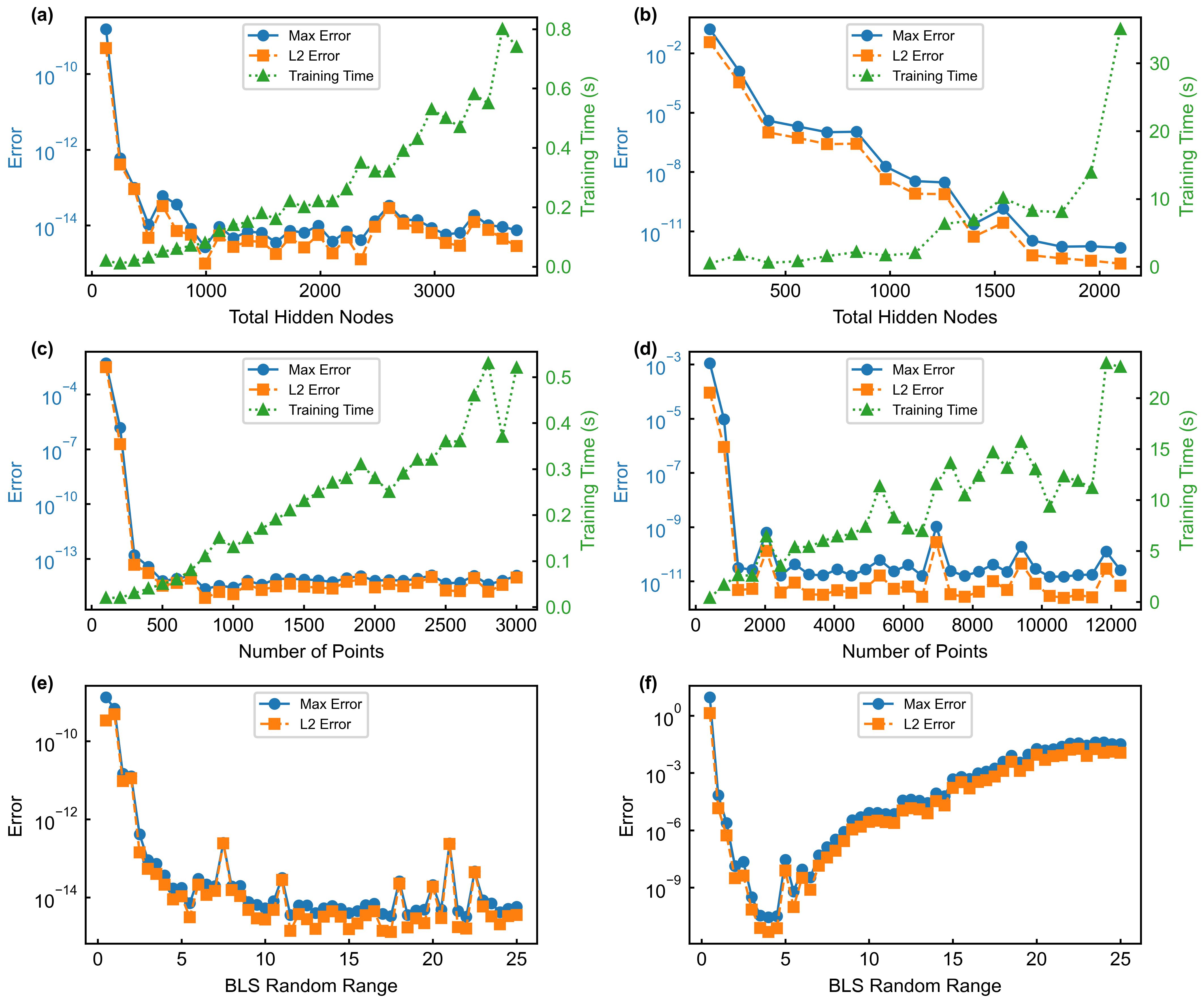} 
   \caption{\textbf{Ablation study of PIBLS hyperparameters.} The left column (\textbf{a}, \textbf{c}, \textbf{e}) corresponds to the linear problem (TC-1), and the right column (\textbf{b}, \textbf{d}, \textbf{f}) corresponds to the nonlinear problem (TC-9). \textbf{a}, \textbf{b} Impact of the total number of hidden nodes on error metrics (left y-axis) and Training Time (green line, right y-axis). \textbf{c}, \textbf{d} Impact of the number of training points on error metrics (left y-axis) and Training Time (green line, right y-axis). \textbf{e}, \textbf{f} Impact of the BLS random initialization range on Maximum Error (blue circles) and $L_2$ Error (orange squares).}
   \label{fig:ablation_study}
\end{figure}
\subsection{Parameter sensitivity analysis}
To assess the robustness of the PIBLS framework and provide practical guidance for its implementation, we investigated its sensitivity to three key hyperparameters: the number of hidden nodes $N_{\text{nodes}}$ (where $N_{\text{nodes}} = n \cdot k + m \cdot q$), the number of training points $N$ (where $N = N_f + N_\text{bc} + N_\text{ic}$), and the random weight initialization range ($R_m$). We conducted this analysis using two representative benchmarks, the linear TC-1 and the nonlinear TC-9. The results are presented in Fig.~\ref{fig:ablation_study}.

\subsubsection{Impact of network architecture}
We first examined the impact of the total number of hidden nodes $N_{\text{nodes}}$ on solution accuracy and training time. Fig.~\ref{fig:ablation_study}a shows that for the linear TC-1, the solution error declines rapidly as the network widens, eventually stabilizing at a plateau below $10^{-14}$ when the number of nodes reaches approximately 1000. As illustrated in Fig.~\ref{fig:ablation_study}b, the nonlinear TC-9 exhibits a similar trend, although it requires a wider network, with the $L_2$ error stabilizing below $10^{-11}$ after approximately 1500 nodes. As indicated by the green curves, increasing the node count beyond these respective saturation points yields only marginal accuracy gains while incurring a near-linear increase in computational cost. These results suggest that selecting a moderately wide network strikes an optimal balance between minimizing computational time and maximizing numerical precision.

\subsubsection{Impact of training point number}
Next, we evaluated the influence of the number of training points $N$ on model performance. As illustrated in Figs.~\ref{fig:ablation_study}c-d for TC-1 and TC-9, the solution error decreases significantly as $N$ increases, eventually reaching a saturation point. For TC-1, errors converge to a floor of approximately $10^{-15}$ with roughly 1000 training points. TC-9 requires a denser sampling, with errors stabilizing around $10^{-12}$ once approximately 1500 points are utilized. Crucially, while error reduction saturates beyond these thresholds, the training time exhibits a general upward trend with $N$ in both scenarios. This behavior underscores a trade-off where increasing $N$ further yields negligible precision improvements while incurring a direct penalty in computational cost. Therefore, selecting a moderate number of training points allows the framework to achieve optimal efficiency without compromising solution accuracy.

\subsubsection{Impact of PIBLS randomization parameters}

Finally, we analyzed the sensitivity to the random weight initialization range $R_m$. The results reveal a clear distinction between linear and nonlinear problems. Fig.~\ref{fig:ablation_study}e demonstrates high robustness for TC-1, where the solution error remains consistently low for any $R_m$ exceeding 3.0. Conversely, for TC-9 shown in Fig.~\ref{fig:ablation_study}f, we observe that optimal performance is constrained to a range of moderate $R_m$ values. Aligning with observations in prior randomized neural network studies~\cite{LocalELM}, initialization values that are too large or too small may lead to suboptimal representations. This indicates that while linear problems offer substantial parameter flexibility, selecting an $R_m$ within an effective moderate interval is critical for maximizing performance in nonlinear problems.

\section{Conclusion}
In this article, a novel PIBLS framework was developed to solve linear and nonlinear PDEs by reformulating the solving process as a robust least-squares optimization. By replacing deep, gradient-based architectures with BLS, the framework constructs a global basis for the solution space and decouples feature generation from weight solving. Numerical experiments demonstrate that PIBLS attains machine-level precision at speeds one to three orders of magnitude faster than conventional physics-informed neural networks, and a rigorous mathematical proof confirms its universal approximation capability within Sobolev spaces. we acknowledge that while linear problems are robust, the accuracy in nonlinear cases remains sensitive to the initialization range of the random parameters.

In the future, we hope to explore adaptive strategies for automatically identifying optimal initialization configurations will be developed to minimize user intervention and ensure maximum accuracy in nonlinear dynamics.

\section*{Acknowledgments}

This study was sponsored by National Natural Science Foundation of China No. 62572199, 92467109, 62476101, U21A20478, and in part by the Major  Key Project of PCL (Grant No. PCL2025A11 and No. PCL2025A13),  in part by the National Key R\&D Program of China 2023YFA1011601, the Guangdong Basic and Applied Basic Research Foundation under Grant 2024A1515140137.
\bibliographystyle{unsrt}  
\bibliography{references}

\end{document}


\maketitle

\begin{abstract}
Partial differential equations (PDEs) play a central role in modeling complex physical, biological, and engineering systems.  Recent advances in Physics-Informed Neural Networks (PINNs) have provided a flexible approach for solving PDEs by embedding PDE constraints into deep learning frameworks, but suffer from slow convergence and unstable gradient-based training. We propose the Physics-Informed Broad Learning System (PIBLS), a backpropagation-free framework that reformulates PDE solving as a least-squares optimization within a Broad Learning System (BLS) architecture. A mathematical proof further establishes the universal approximation property of PIBLS for PDEs. Experiments on linear and nonlinear PDEs demonstrate that PIBLS is one to three orders of magnitude faster than conventional PINNs while achieving significantly higher solution accuracy. Furthermore, PIBLS effectively models real-world biological dynamics such as cell migration, showcasing strong generalization from limited data. PIBLS provides a powerful, fast, and practical alternative to both deep learning solvers and traditional numerical methods.

\end{abstract}

\section{Introduction}

\section*{Acknowledgments}
This study was sponsored by National Natural Science Foundation of China No. 92467109 and in part by National Natural Science Foundation of China 62476101, in part by the National Key R\&D Program of China 2023YFA1011601, and in part by the Major Key Project of PCL, China under Grant PCL2025A11. 

\bibliographystyle{unsrt}  
\bibliography{references}